\documentclass[]{style}

\usepackage{titletoc}
\usepackage[toc,page,header]{appendix}

\usepackage{amsmath}
\usepackage{amssymb}
\usepackage{mathtools}
\usepackage{mathrsfs}

\usepackage{graphicx}
\usepackage{adjustbox}
\usepackage{float}
\usepackage{multirow}
\usepackage{multicol}
\usepackage{tabularx}
\usepackage{colortbl}
\usepackage{subcaption}
\usepackage{wrapfig}
\usepackage{placeins}

\usepackage{enumitem}
\usepackage{algorithm}
\usepackage{algpseudocode}
\usepackage{pifont}

\newcommand{\cmark}{\textcolor{green!60!black}{\ding{51}}}
\newcommand{\xmark}{\textcolor{red!70!black}{\ding{55}}}

\usepackage{cleveref}

\makeatletter
\@ifpackageloaded{ulem}{\normalem}{}
\makeatother

\newcommand{\comicfont}{\sffamily\bfseries}
\newtheorem{proposition}{Proposition}

\DeclareRobustCommand{\emailicon}{%
  \raisebox{-0.15em}{\includegraphics[height=1em]{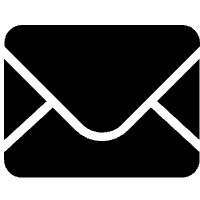}}%
}

\title{Self-OPD: On-Policy Distillation for Flow Matching Models without Teacher}

\author[1,3,*]{Shiyi Zhang}
\author[2,3,*,\boxtimes]{Mushui Liu}
\author[2]{Yunze Tong}
\author[3]{Wanggui He}
\author[3]{Siyu Zou}
\author[3]{Jinlong Liu}
\author[2]{Yunlong Yu}
\author[1]{Jian Song}
\author[3,\boxtimes]{Hao Jiang}
\author[3]{Pipei Huang}
\author[3]{Bo Zheng}

\affiliation[1]{Tsinghua University}
\affiliation[2]{Zhejiang University}
\affiliation[3]{Alibaba Group}

\contribution[]{%
  \emailicon\,\email{shiyi-zh24@mails.tsinghua.edu.cn}
  \emailicon\,\email{lms@zju.edu.cn}
}

\contribution[*]{Equal contribution}
\contribution[\boxtimes]{Corresponding author}

\abstract{
On-policy distillation (OPD), which leverages a pre-trained, specialized teacher model to provide dense supervisory signals, has achieved significant success in Large Language Models (LLMs) and has recently been adapted to flow matching models.
However, this paradigm suffers from two major issues: First, training a separate, task-specific teacher for every new objective incurs high computational costs. Second, the discrepancy between teacher and student distributions often leads to compounding errors along the generation trajectory.
In this paper, we introduce \textbf{Self-OPD}, a teacher-free OPD framework for flow matching models that turns the student's own self-exploration into step-wise supervision.
At each timestep, Self-OPD branches the deterministic next-state prediction into $K$ stochastic SDE candidates, rolls them out with the ODE sampler, and compares their rewards against a deterministic self-reference baseline to obtain normalized advantages.
The velocity field is optimized with an all-branch pull-push objective, where high-advantage branches attract the student and low-advantage branches repel it under direction-aware attenuation and SDE-variance normalization.
For multi-objective alignment, Self-OPD fuses normalized scores at the reward level, avoiding direct gradient conflict.
Experiments on single and mixed reward benchmarks show that Self-OPD outperforms prior RL and OPD methods without task-specific teachers.
}

\checkdata[
\raisebox{-0.25em}{\includegraphics[height=1.1em]{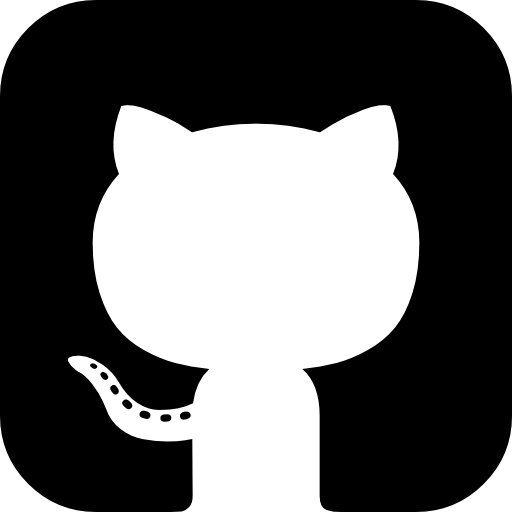}}~~Github]{\href{https://github.com/Shiy-Zhang/Self-OPD}{\texttt{https://github.com/Shiy-Zhang/Self-OPD}} 
\\[-1.5ex]}

\begin{document}
\maketitle


\section{Introduction}
Flow Matching (FM) models~\cite{flow_matching, rectified_flow, SD3}, which learn a continuous velocity field that transports noise to data, have become a dominant backbone for high-quality visual generation~\cite{DDPM, LDM, SDXL, FLUX}.
Despite their generative strength, aligning FM models with downstream objectives such as text rendering~\cite{TextDiffuser, GlyphControl}, compositional correctness, and human preference remains challenging.

\begin{figure}[!t]
  \vspace{-5mm}
  \centering
  \captionsetup{skip=2pt}
  \begin{minipage}{0.88\linewidth}
    \centering
    \begin{minipage}[c]{0.56\linewidth}
      \centering
      \resizebox{\linewidth}{!}{%
        \scriptsize
        \begin{tabular}{lccc}
          \toprule
            & Flow-GRPO & Teacher-based OPD & \textbf{Self-OPD} \\
          \midrule
          Teacher-free       & \cmark & \xmark & \cmark \\
          Per-step supv.     & \xmark & \cmark & \cmark \\
          On-policy data     & \cmark & \cmark & \cmark \\
          Credit assign.     & Terminal & Dense & Dense \\
          Adapts to student  & \cmark & \xmark & \cmark \\
          Exploration & Full traj. & Teacher & Local \\
          \bottomrule
        \end{tabular}%
      }
    \end{minipage}%
    \hfill
    \begin{minipage}[c]{0.43\linewidth}
      \centering
      \includegraphics[width=\linewidth]{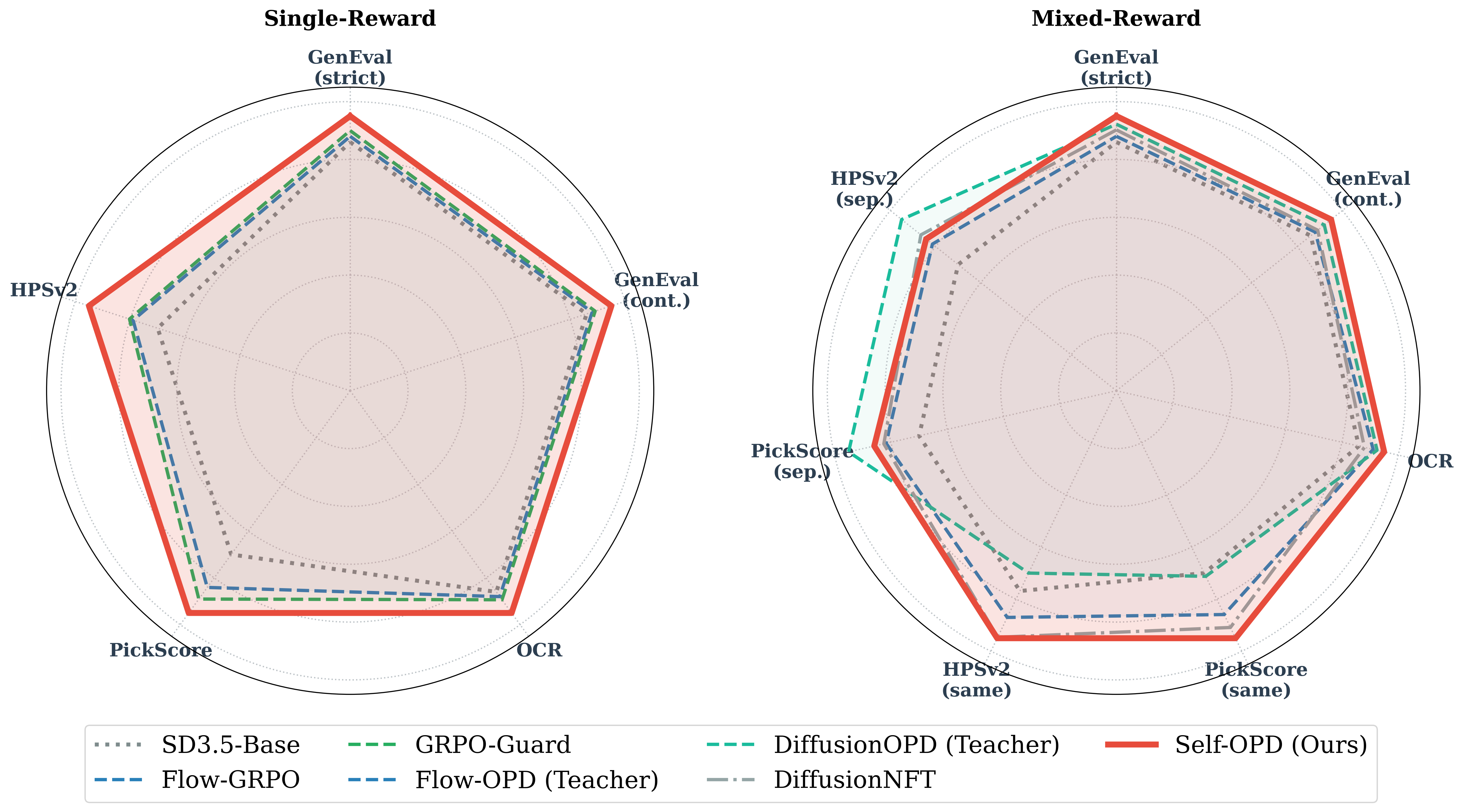}
    \end{minipage}

    \vspace{1pt}
    \begin{minipage}[t]{0.56\linewidth}
      \centering\scriptsize\comicfont (a) Comparison of alignment paradigms
    \end{minipage}%
    \hfill
    \begin{minipage}[t]{0.43\linewidth}
      \centering\scriptsize\comicfont (b) Single- and mixed-reward results
    \end{minipage}

    \vspace{0.5mm}
    \includegraphics[width=\linewidth]{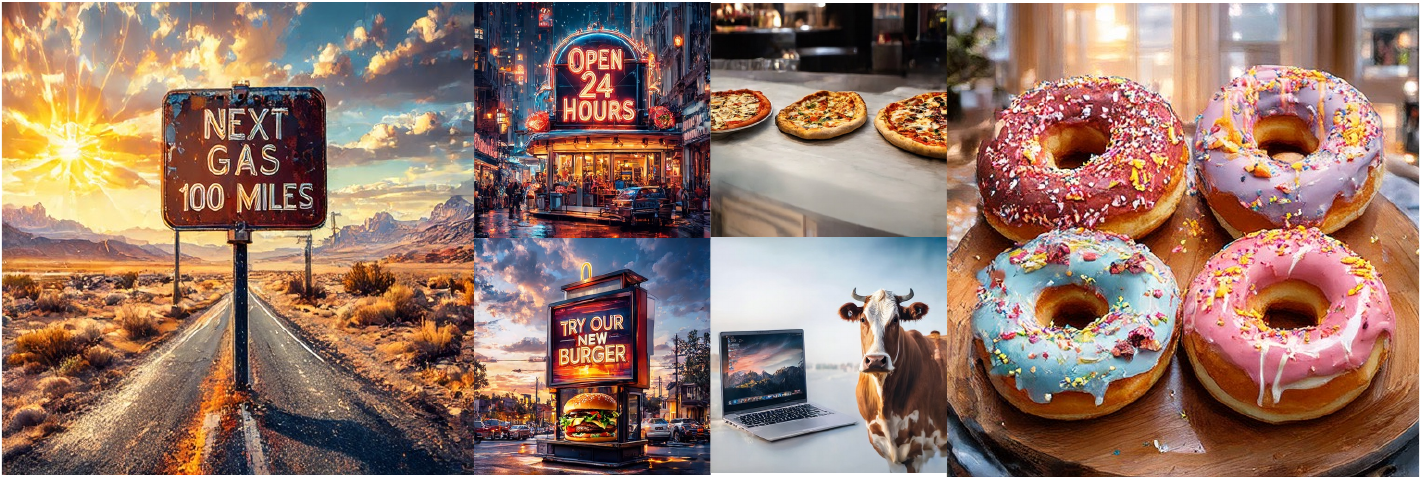}
    \par\vspace{-1pt}
    {\scriptsize\comicfont (c) Generated samples from a single Self-OPD model}
  \end{minipage}
  \caption{\textbf{Overview of Self-OPD.}}
  \label{fig:teaser}
\end{figure}

Existing reinforcement learning (RL) approaches~\cite{Flow-GRPO, DiffusionNFT, DDPO, DPOK, DiffusionDPO, GRPO-Guard} directly optimize reward signals by treating generation as a sequential decision process.
As summarized in Fig.~\ref{fig:teaser}(a) and illustrated in Fig.~\ref{fig:paradigm_comparison}, these methods can be teacher-free, but they often rely on terminal scores from full denoising trajectories.
Consequently, reward credit must be assigned backward through many steps, which leads to high-variance gradients and makes multi-objective alignment fragile when different rewards favor different visual properties.

On-policy distillation (OPD) provides a more stable alternative by supplying dense, step-wise supervision on the student's own trajectory.
Recent OPD methods for FM or diffusion models~\cite{flow_opd, diffusion_opd} regress the student velocity toward a pretrained teacher prediction at each denoising step, thereby improving training stability and sample efficiency.
Yet this teacher-dependent design introduces a different bottleneck: each new objective requires training or acquiring a specialized teacher, the student is upper-bounded by the quality and bias of that teacher, and field-level fusion of multiple teachers can create conflicting update directions~\cite{flow_opd, diffusion_opd, dance_opd}.
In particular, some OPD methods handle multiple objectives by \emph{decoupling} the problem into separately trained, task-specific teachers whose velocity fields are then routed and merged at the field level~\cite{diffusion_opd}, which is fundamentally at odds with our aim of driving a \emph{single} image to high reward across all objectives at once.
This motivates a central question: \textit{\textbf{Can we retain the dense per-step supervision of OPD while removing the need for any external teacher?}}

We answer this question with \textbf{Self-OPD}, a teacher-free on-policy distillation framework for flow matching models.
The key idea is to replace teacher-provided velocity targets with reward-weighted targets discovered from the student's own local exploration.
At each on-policy timestep $t_j$, Self-OPD branches the deterministic next-state prediction $x_{t_{j+1},\theta}$ into $K$ stochastic SDE candidates, rolls each candidate out deterministically to a clean image, and evaluates the resulting images with task rewards.
In parallel, a deterministic ODE rollout from the same parent state provides a self-reference baseline, allowing terminal rewards to be converted into normalized branch advantages.
Self-OPD then performs all-branch pull-push distillation: high-advantage branches pull the velocity field toward better local trajectories, while low-advantage branches push it away from poor directions.
A direction-aware attenuation term stabilizes this signed objective by suppressing repulsion that would counteract the best branch, and an SDE-variance normalization yields a principled per-step alignment loss connected to reward-tilted KL minimization.
This self-referenced formulation also changes how multi-objective alignment is handled.
Instead of differentiating through a weighted sum of potentially conflicting objective losses, Self-OPD fuses normalized scalar scores at the reward level and uses the composite reward only to rank sampled branches.
The actual regression target remains a concrete trajectory velocity sampled from the student's own neighborhood, which naturally supports black-box rewards and avoids direct gradient competition among objectives.
Rather than partitioning capabilities across specialized teachers as in field-level decoupling, Self-OPD selects branches inside the \emph{joint} high-reward region, so that a single image is directly optimized to score high on all rewards at once.
As shown in Fig.~\ref{fig:teaser}, this lets a \emph{single} generated image simultaneously satisfy multiple objectives---accurate text rendering, correct composition, and high aesthetic quality---rather than having different images each excel under a different reward; we return to this joint high-reward behavior in Sec.~\ref{sec:fusion}.

Overall, our contributions are:
\begin{itemize}
    \item We propose \textbf{Self-OPD}, a teacher-free on-policy distillation framework for flow matching models that converts reward-guided self-exploration and a self-reference baseline into dense step-wise supervision.
    \item We introduce an all-branch pull-push distillation objective with SDE-variance normalization and direction-aware attenuation, using both high- and low-advantage branches while maintaining stable optimization.
    \item We develop a reward-level fusion strategy for multi-objective alignment, enabling black-box composition without field-level teacher routing or gradient conflicts.
    \item We show empirically that Self-OPD outperforms prior RL and teacher-based OPD methods across single-reward and mixed-reward benchmarks.
\end{itemize}

\begin{figure*}[t]
\centering
\includegraphics[width=\linewidth]{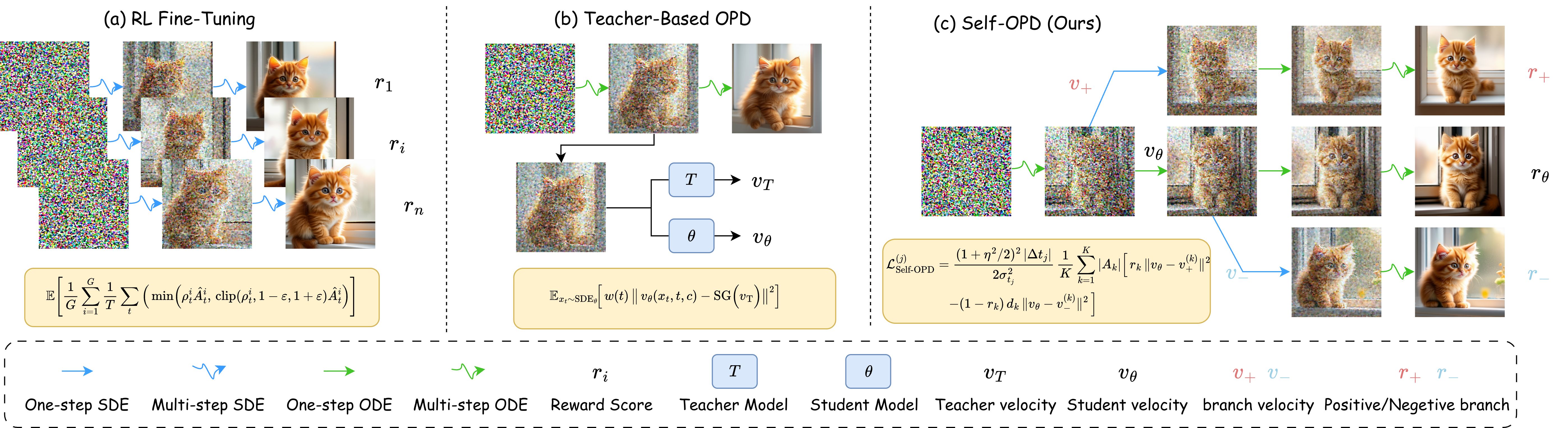}
\caption{\textbf{Comparison of alignment paradigms.}  (a)~\textbf{Flow-GRPO} relies on trajectory-level policy gradients from terminal rewards. (b)~\textbf{Flow-OPD} requires a pretrained teacher for step-wise MSE supervision. (c)~\textbf{Self-OPD (Ours)} turns local stochastic branches and a self-reference baseline into dense pull-push supervision, without requiring any teacher.}
\label{fig:paradigm_comparison}
\end{figure*}

\section{Related Work}
\label{sec:related-work}
\textbf{RL in Diffusion Models.}
RL has recently emerged as an effective paradigm for aligning diffusion-based visual generation models~\cite{LDM, SDXL, SD3, FLUX, liu2024llm4gen, TFCustom} with human preferences~\cite{Flow-GRPO, DiffusionNFT, AWM, DDPO, DPOK, DiffusionDPO, DRaFT, ImageReward,  ImageReward}. From the perspective of backward estimation, FlowGRPO~\cite{Flow-GRPO}, TempFlow-GRPO~\cite{TempFlow-GRPO}, TP-GRPO~\cite{TP-GRPO}, and DanceGRPO~\cite{DanceGRPO} reformulate denoising as a sequential decision-making problem, enabling policy gradient optimization over discrete diffusion timesteps. DiffusionNFT~\cite{DiffusionNFT} instead operates on the forward diffusion trajectory, treating the noising process as an environment to derive reward signals for policy updates. More recently, several works~\cite{GRPO-Guard,DiffusionNFT} have explored multi-reward optimization strategies to simultaneously improve generation quality across diverse criteria while preserving generalization, establishing a multi-task RL learning framework for visual generation.

\textbf{On-Policy Distillation.}
On-policy distillation (OPD) ~\cite{gkd2024, flow_opd, diffusion_opd, visionopd, collectionlora} has demonstrated notable progress in large language models (LLMs), owing to its ability to mitigate exposure bias and improve training efficiency through student-generated sequences~\cite{gkd2024,restem2023}.
GKD~\cite{gkd2024} established the canonical OPD framework by training the student model on its own generated outputs, effectively bridging the distribution gap between teacher and student. Building upon this, Flow-OPD~\cite{flow_opd} and Diffusion-OPD~\cite{diffusion_opd} extend OPD~\cite{Poly-OPD,Any-OPD} to flow matching and diffusion models, respectively, achieving competitive generation quality with significantly reduced sampling steps~\cite{DMD, DMD2}. D-OPSD~\cite{d-opsd} further leverages multi-modal features to distill rich cross-modal knowledge into a compact single-stream textual representation, enhancing semantic alignment in text-to-image generation.
In contrast, our work extends OPD to a multi-task setting, enabling simultaneous optimization across diverse generation objectives without relying on a fixed teacher model.

\section{Methodology}

\begin{figure*}[!t]
    \centering
    \includegraphics[width=1.0\linewidth]{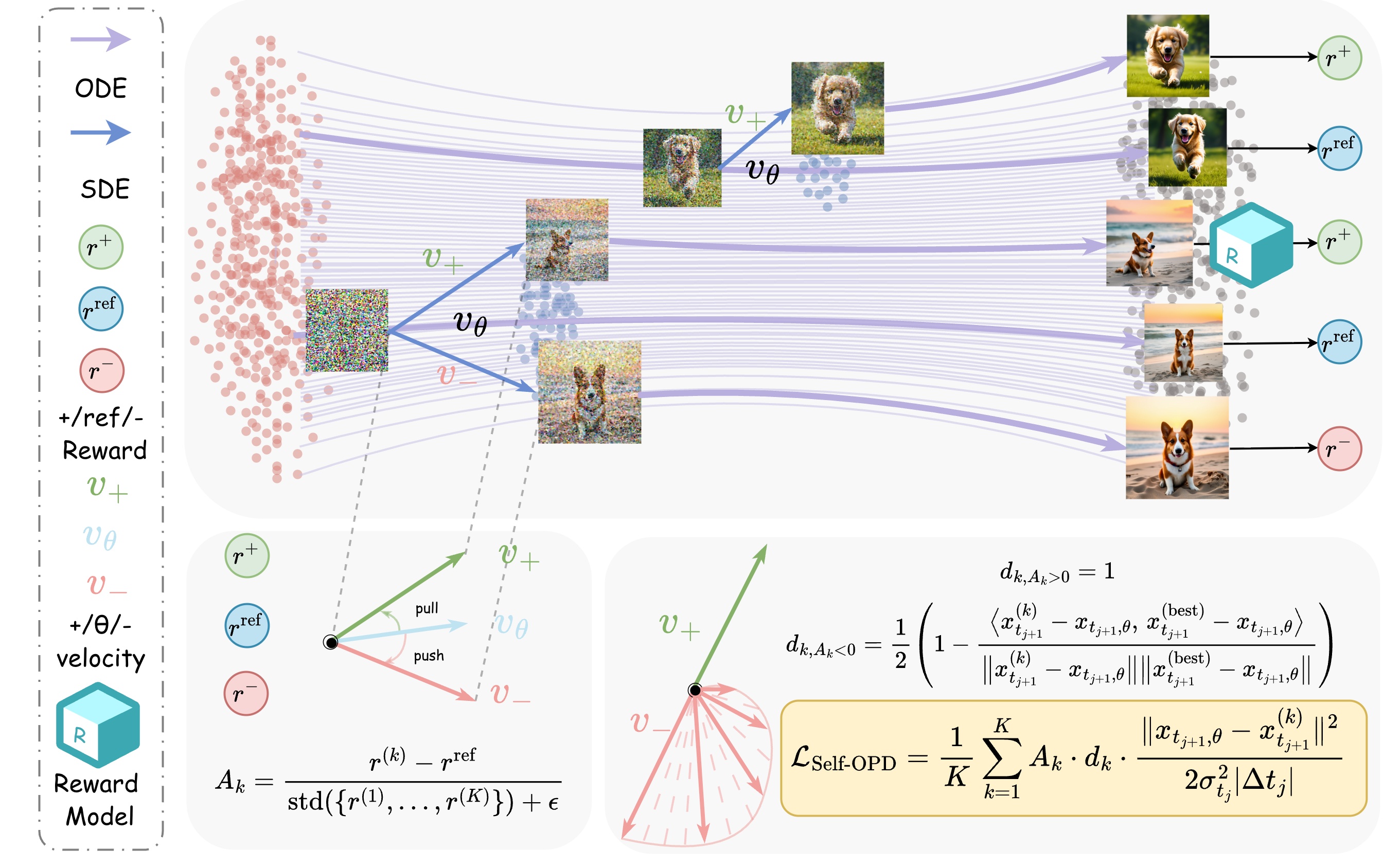}
    
    \caption{\textbf{Self-OPD pipeline.} 
  \textit{Top:} The student branches its prediction into $K$ SDE paths (blue), generates images via ODE rollouts (purple), and scores them ($r^+, r^{\mathrm{ref}}, r^-$). 
  \textit{Bottom-left:} Rewards define a self-referenced advantage $A_k$ that pulls toward high-reward velocities $v_+$ and pushes away from low-reward velocities $v_-$. 
  \textit{Bottom-right:} A direction-aware coefficient $d_k$ gates the push to prevent it from counteracting the pull, resulting in the multi-branch loss $\mathcal{L}_{\text{Self-OPD}}$.}
    \label{fig:architecture}
\end{figure*}

\subsection{Preliminaries}
\label{sec:prelim}

\textbf{Flow Matching and SDE Formulation.} 
Flow Matching (FM)~\cite{flow_matching, rectified_flow} learns a continuous velocity field $v_\theta$ that transports a noise distribution $p_1 = \mathcal{N}(0, \mathbf{I})$ to the data distribution $p_0$ along straight-line trajectories. Formally, the probability path is defined as $x_t = (1-t)x_0 + t\epsilon$ for $t \in [0, 1]$, where $\epsilon \sim \mathcal{N}(0, \mathbf{I})$. The model parameterizes $v_\theta(x_t, t)$ by minimizing the flow matching objective:
\begin{equation}
    \mathcal{L}_{\mathrm{FM}} = \mathbb{E}_{t, x_0, \epsilon}\left[\|v_\theta(x_t, t) - (\epsilon - x_0)\|^2\right].
    \label{eq:flow_matching}
\end{equation}
During inference, samples are generated by integrating the reverse-time ordinary differential equation (ODE) $\frac{\mathrm{d}x_t}{\mathrm{d}t} = v_\theta(x_t, t)$ backward from $t_0 = 1$ to $t_S \approx 0$. Specifically, we discretize the trajectory into $S$ integration steps along a discrete schedule $\{t_j\}_{j=0}^{S}$ satisfying $1 = t_0 > t_1 > \dots > t_S \approx 0$. The infinitesimal transition step is denoted as $\Delta t_j = t_{j+1} - t_j < 0$ for $j \in \{0, \dots, S-1\}$.

To enable stochastic exploration, the reverse ODE can be augmented into a reverse-time stochastic differential equation (SDE)~\cite{score_sde, flow_opd}. Under the Euler-Maruyama discretization scheme, each transition step decomposes into a deterministic next-state prediction and an isotropic stochastic perturbation:
\begin{equation}
    x_{t_{j+1}} = x_{t_{j+1},\theta} + \sigma_{t_j} \sqrt{|\Delta t_j|}\, z_j, \quad z_j \sim \mathcal{N}(0, \mathbf{I}),
    \label{eq:sde_step}
\end{equation}
where $\sigma_{t_j} = \eta\sqrt{t_j/(1{-}t_j)}$ represents the noise schedule parameterized by the coefficient $\eta \in [0, 1]$. The drift of this reverse SDE is governed by the score function $\nabla_{x_t}\log p_t(x_t)$, which is inherently constrained by the velocity field. Conditioned on a data point $x_0$, the conditional distribution is $x_t \mid x_0 \sim \mathcal{N}\bigl((1{-}t)x_0,\, t^2 \mathbf{I}\bigr)$, which yields the conditional score $\nabla_{x_t}\log p_t(x_t\mid x_0) = -\frac{x_t-(1{-}t)x_0}{t^2} = -\frac{\epsilon}{t}$. Averaging over the posterior $p(x_0\mid x_t)$ establishes the marginal-score identity $\nabla_{x_t}\log p_t(x_t) = -\frac{\mathbb{E}[\epsilon\mid x_t]}{t}$. By eliminating $x_0 = \frac{x_t - t\epsilon}{1-t}$ from the parameterized velocity target $v_\theta \approx \mathbb{E}[\epsilon - x_0 \mid x_t]$, we obtain the conditional expectation $\mathbb{E}[\epsilon\mid x_t] = x_t + (1{-}t)v_\theta$, which simplifies the score function to:
\begin{equation}
    \nabla_{x_t}\log p_t(x_t) \approx -\frac{x_t + (1{-}t)\,v_\theta(x_t, t)}{t}.
    \label{eq:score}
\end{equation}
The reverse SDE transition follows the score-augmented drift $v_\theta - \frac{\sigma_{t_j}^2}{2}\nabla_{x_t}\log p_t$~\cite{score_sde}. Substituting Eq.~\ref{eq:score} into this drift formulation scales the correction term to $\frac{\sigma_{t_j}^2}{2t_j}\bigl(x_{t_j}+(1{-}t_j)v_\theta\bigr)$. Consequently, the deterministic next-state prediction $x_{t_{j+1},\theta}$ is formulated as:
\begin{equation}
    x_{t_{j+1},\theta} = x_{t_j} + \left[ v_\theta(x_{t_j}, t_j) + \frac{\sigma_{t_j}^2}{2t_j} \Bigl( x_{t_j} + (1{-}t_j)\,v_\theta(x_{t_j}, t_j) \Bigr) \right] \Delta t_j.
    \label{eq:mu_theta}
\end{equation}
Setting $\eta = 0$ naturally recovers the deterministic Euler ODE step $x_{t_j} + v_\theta \Delta t_j$. To explicitly reveal the dependence of the prediction on the learned velocity, we collect the $v_\theta$ terms in Eq.~\ref{eq:mu_theta}, yielding the coefficient $\Delta t_j\bigl[1+\frac{\sigma_{t_j}^2}{2t_j}(1{-}t_j)\bigr]$. Substituting the schedule $\sigma_{t_j}^2=\eta^2 t_j/(1{-}t_j)$ cancels the terms $t_j$ and $(1{-}t_j)$, leaving a timestep-independent constant $\bigl(1+\frac{\eta^2}{2}\bigr)$. Hence, the deterministic prediction displays an elegant affine relationship with the velocity field:
\begin{equation}
    x_{t_{j+1},\theta} = b_{t_j}(x_{t_j}) + c_{t_j}\, v_\theta(x_{t_j}, t_j), \qquad
    c_{t_j} = \bigl(1+\tfrac{\eta^2}{2}\bigr)\Delta t_j,
    \label{eq:affine_mean}
\end{equation}
where $b_{t_j}(x_{t_j})=\bigl(1+\frac{\sigma_{t_j}^2\Delta t_j}{2t_j}\bigr)x_{t_j}$ aggregates all components independent of $v_\theta$. This affine equivalence guarantees that any distillation objective formulated in transition space can be minimized with identical convergence properties in velocity space.

In this work, we propose \textbf{Self-OPD}, a teacher-free on-policy distillation framework that transforms self-generated rewards into dense, step-wise supervision. As illustrated in Fig.~\ref{fig:architecture}, Self-OPD first explores local trajectories around the student's on-policy path via stochastic SDE branching, evaluates these candidates against a deterministic self-reference baseline, and finally distills the collective feedback through an advantage-weighted pull-push objective.

\subsection{Self-Referenced Evaluation and Self-Exploration}
To retain dense supervision without a pretrained teacher, Self-OPD lets the student evaluate its own local alternatives. Given an on-policy latent $x_{t_j}$, the method first samples $K$ stochastic next states around the deterministic prediction and then compares their final rewards against the student's default ODE trajectory from the same state. This local, self-referenced comparison turns terminal rewards into low-variance per-step guidance.

\textbf{SDE Branching Exploration.}
At timestep $t_j$, we compute the deterministic next-state prediction $x_{t_{j+1},\theta}$ with a single forward pass through the student transformer. Following the SDE transition in Eq.~\ref{eq:sde_step}, we instantiate $K$ candidate branches by drawing independent Gaussian perturbations:
\begin{equation}
    x_{t_{j+1}}^{(k)} = x_{t_{j+1},\theta} + \sigma_{t_j}\sqrt{|\Delta t_j|}\, z_k,\quad
    z_k \sim \mathcal{N}(0, I)
    \label{eq:branching}
\end{equation}
where $\quad k = 1,\ldots,K$. All branches share the same base prediction $x_{t_{j+1},\theta}$, so the expensive transformer evaluation is performed only once; the additional diversity comes from lightweight SDE perturbations. The noise level $\sigma_{t_j}$, set by $\eta = 0.6$ in our experiments, controls the exploration radius.

\textbf{ODE Rollout and Self-Reference Evaluation.}
Each branch $x_{t_{j+1}}^{(k)}$ is completed to a clean latent $\hat{x}_0^{(k)}$ by a deterministic ODE rollout ($\eta{=}0$), decoded to pixel space, and scored by a task-specific reward model $R$:
\begin{equation}
    r^{(k)} = R\left(\mathrm{Dec}(\hat{x}_0^{(k)}),\, c\right), \quad k = 1, \ldots, K,
    \label{eq:scoring}
\end{equation}
where $c$ is the prompt and $\mathrm{Dec}(\cdot)$ is the VAE decoder. To obtain a variance-reducing baseline without a teacher, we also run a fully deterministic ODE trajectory directly from the parent state $x_{t_j}$ and denote its reward as $r^{\mathrm{ode}}$. The branch advantage is then normalized relative to this self-reference:
\begin{equation}
    A_k = \frac{r^{(k)} - r^{\mathrm{ode}}}{\mathrm{std}(\{r^{(1)}, \ldots, r^{(K)}\}) + \epsilon},
    \label{eq:advantage}
\end{equation}
where positive advantages indicate local directions that outperform the student's default trajectory, while negative advantages identify directions to avoid. For multi-reward training, $r^{(k)}$ is the composite reward defined in Eq.~\ref{eq:composite_reward}; it is still used only as a non-differentiable branch-ranking signal inside Eq.~\ref{eq:advantage}.

\subsection{All-Branch Pull-Push Distillation}
\label{sec:distillation}
Instead of regressing only to the best branch, Self-OPD uses the full exploration neighborhood. Positive-advantage branches should pull the deterministic prediction toward locally better trajectories, while negative-advantage branches should push it away from poor directions. This all-branch update reduces target switching and preserves the information contained in branches that Best-of-$K$ selection would discard.

Not every negative branch is equally useful. If a low-reward branch points in nearly the same direction as the best branch, a strong repulsive update would counteract the desirable pull. We therefore introduce a direction-aware attenuation coefficient:
\begin{equation}
    d_k = \begin{cases}
        1 & \text{if } A_k \geq 0, \\[6pt]
        \dfrac{1}{2} \left( 1 - \dfrac{\langle \delta_k, \delta_{\mathrm{best}} \rangle}{\| \delta_k \| \| \delta_{\mathrm{best}} \|} \right) & \text{if } A_k < 0,
    \end{cases}
    \label{eq:direction}
\end{equation}
where $\delta_k = x_{t_{j+1}}^{(k)} - x_{t_{j+1},\theta}$ and $\delta_{\mathrm{best}} = x_{t_{j+1}}^{(\mathrm{best})} - x_{t_{j+1},\theta}$ are measured from the deterministic prediction. The coefficient is bounded in $[0,1]$: it keeps the full repulsion for negative branches that point opposite to the best direction, and suppresses repulsion when a negative branch is aligned with the best branch.

Because $x_{t_{j+1},\theta}$ is affine in $v_\theta$ by Eq.~\ref{eq:affine_mean}, every branch can be associated with an effective branch velocity $v^{(k)}$ satisfying $x_{t_{j+1}}^{(k)} = b_{t_j}(x_{t_j}) + c_{t_j}v^{(k)}$. Subtracting Eq.~\ref{eq:affine_mean} from the branch definition Eq.~\ref{eq:branching} and dividing by $c_{t_j}=(1+\eta^2/2)\Delta t_j$ gives the closed form
\begin{equation}
    v^{(k)} = v_\theta - \frac{\sigma_{t_j}}{(1+\eta^2/2)\sqrt{|\Delta t_j|}}\, z_k,
    \label{eq:branch_velocity}
\end{equation}
so injecting Gaussian noise in transition space is exactly a perturbation of the velocity target. Consequently the displacement factors as $x_{t_{j+1},\theta}-x_{t_{j+1}}^{(k)} = c_{t_j}\,(v_\theta - v^{(k)})$, so the squared transition distance equals the squared velocity distance scaled by $c_{t_j}^2=(1+\eta^2/2)^2|\Delta t_j|^2$. Dividing this by the transition-kernel variance normalization $2\sigma_{t_j}^2|\Delta t_j|$ turns the constant into the prefactor $(1+\eta^2/2)^2|\Delta t_j|/(2\sigma_{t_j}^2)$ below. Let $r_k=\mathbf{1}[A_k\geq 0]$, and define $v_+^{(k)}$ and $v_-^{(k)}$ as the effective velocities of positive and negative branches, respectively. The step-wise Self-OPD objective is:
\begin{equation}
\begin{aligned}
    \mathcal{L}_{\text{Self-OPD}}^{(j)} = \frac{(1+\eta^2/2)^2\,|\Delta t_j|}{2\sigma_{t_j}^2}\, \frac{1}{K}\sum_{k=1}^{K} |A_k|\Big[\, r_k\,\|v_\theta - v_+^{(k)}\|^2 
    - (1-r_k)\,d_k\,\|v_\theta - v_-^{(k)}\|^2 \,\Big],
\end{aligned}
\label{eq:loss_unified}
\end{equation}
The prefactor follows from the SDE transition variance and the affine change of variables from transition mean to velocity, matching the normalization used in OPD-style per-step regression~\cite{flow_opd,diffusion_opd}. The first term pulls $v_\theta$ toward high-reward branch velocities, while the second term pushes it away from low-reward velocities after direction-aware attenuation.

\textbf{Connection to Per-Step KL.}
The normalization above also admits a KL interpretation. Let $q_\theta(\cdot \mid x_{t_j})=\mathcal{N}(x_{t_{j+1},\theta}, \sigma_{t_j}^2|\Delta t_j|I)$ denote the SDE transition kernel, and let $q^*(\cdot \mid x_{t_j}) \propto q_\theta(\cdot \mid x_{t_j})\exp(A(\cdot)/\tau)$ be a reward-tilted target distribution, as in reward-weighted regression.
\begin{proposition}[Per-step KL gradient]
\label{prop:kl}
For fixed $x_{t_j}$ and fixed target $q^*$, the gradient of the reverse KL with respect to the student transition mean is
\begin{equation}
    \nabla_{x_{t_{j+1},\theta}} D_{\mathrm{KL}}(q^* \| q_\theta)
    = \frac{1}{\sigma_{t_j}^2|\Delta t_j|}
    \mathbb{E}_{x_{t_{j+1}}\sim q^*}\!\left[x_{t_{j+1},\theta}-x_{t_{j+1}}\right].
    \label{eq:kl_gradient}
\end{equation}
\end{proposition}
Self-OPD approximates this expectation with the $K$ sampled branches, using $A_kd_k$ as the signed importance weight of each branch; we prove the Proposition next and then use it to explain the loss.

\textbf{Proof of Prop.~\ref{prop:kl}.}
Abbreviate the transition mean $\mu_\theta = x_{t_{j+1},\theta}$ and the isotropic covariance $\Sigma = \sigma_{t_j}^2|\Delta t_j|\,I$. Since $q^*$ is the fixed regression target, differentiating the reverse KL with respect to $\mu_\theta$ acts only through the $\log q_\theta$ term. Splitting the two terms,
\begin{equation}
    D_{\mathrm{KL}}(q^*\|q_\theta) = \mathbb{E}_{x\sim q^*}[\log q^*(x)] - \mathbb{E}_{x\sim q^*}[\log q_\theta(x)],
\end{equation}
the first is independent of $\mu_\theta$ and vanishes under $\nabla_{\mu_\theta}$. For the Gaussian kernel,
\begin{equation}
    \log q_\theta(x) = -\tfrac{1}{2}(x-\mu_\theta)^{\!\top}\Sigma^{-1}(x-\mu_\theta) - \tfrac{1}{2}\log\!\big((2\pi)^{d}|\Sigma|\big),
\end{equation}
only the quadratic form depends on $\mu_\theta$, so $\nabla_{\mu_\theta}\log q_\theta(x)=\Sigma^{-1}(x-\mu_\theta)$. By linearity of expectation,
\begin{equation}
    \nabla_{\mu_\theta}D_{\mathrm{KL}}(q^*\|q_\theta)
    = -\,\mathbb{E}_{x\sim q^*}\!\big[\Sigma^{-1}(x-\mu_\theta)\big]
    = \Sigma^{-1}\,\mathbb{E}_{x\sim q^*}\!\big[\mu_\theta-x\big],
\end{equation}
and substituting $\Sigma^{-1}=\tfrac{1}{\sigma_{t_j}^2|\Delta t_j|}I$ recovers Eq.~\ref{eq:kl_gradient}. \hfill$\square$

\textbf{From the gradient to the training loss.}
The gradient has a direct operational reading: the KL-optimal update moves the student mean $x_{t_{j+1},\theta}$ toward samples of the reward-tilted target $q^*$ with a step size set by the transition \emph{precision} $1/(\sigma_{t_j}^2|\Delta t_j|)$. Because $q^*$ is intractable, Self-OPD replaces the expectation by a Monte Carlo estimate over the $K$ SDE branches, importance-weighting each by its advantage $A_k$ and, for negative branches, the direction gate $d_k$; rewriting the mean-space target as a velocity-space target through the affine relation Eq.~\ref{eq:affine_mean} then recovers Eq.~\ref{eq:loss_unified}, up to the constant $(1+\eta^2/2)^2$ from the change of variables. The transition-variance normalization is therefore not a free hyperparameter but the precision that makes every per-step regression an unbiased estimate of the same KL gradient; a timestep-independent scale would over- or under-weight steps along the trajectory.

To aggregate the step-wise losses across the entire denoising trajectory, we define the total training objective as:
\begin{equation}
    \mathcal{L}_{\text{total}} = \sum_{j=0}^{S-1} \alpha_j \, \mathcal{L}_{\text{Self-OPD}}^{(j)},
    \label{eq:total_loss}
\end{equation}
where $\mathcal{L}_{\text{Self-OPD}}^{(j)}$ is the step-wise loss at timestep $t_j$ (Eq.~\ref{eq:loss_unified}), and $\alpha_j$ is a timestep-dependent weight. Following the intuition that early-to-mid timesteps (near noise) establish the global semantic layout while late timesteps (near data) merely refine local details, we set $\alpha_j$ to prioritize these early exploration phases, smoothly decaying its value as the trajectory approaches convergence, when $t = 0$.

\subsection{Reward-Level Fusion for Multi-Objective Alignment}
\label{sec:fusion}

Simultaneous multi-objective alignment (e.g., text, aesthetics, preference) is challenging. Previous approaches~\cite{flow_opd, diffusion_opd, dance_opd} typically blend gradients in the shared parameter space $\Theta$ by optimizing $\sum_{m} \lambda_m \mathcal{L}_m(\theta)$. However, conflicting objectives ($\langle \nabla_\theta \mathcal{L}_i, \nabla_\theta \mathcal{L}_j \rangle < 0$) trigger destructive gradient interference, causing a ``seesaw effect'' where improving one metric degrades others. Self-OPD bypasses this by fusing objectives at the non-differentiable \textit{reward level}, shifting the Pareto-front search from the parameter space to the simpler trajectory space:

Let $q_\theta(\cdot | x_{t_j})$ be the step-$j$ SDE transition kernel, and $q^*_m(\cdot | x_{t_j}) \propto q_\theta(\cdot | x_{t_j}) \exp(A_m(\cdot)/\tau)$ be the reward-tilted target for objective $m \in \{1,\dots,M\}$. Field-level fusion minimizes the weighted KL sum $\sum_m \lambda_m D_{\mathrm{KL}}(q^*_m \| q_\theta)$ in parameter space, where objectives that disagree contribute opposing gradients. Consider instead the geometric mean of the per-objective tilts, i.e., the \emph{joint} target
\begin{equation}
\begin{aligned}
    q^*_{\mathrm{joint}}(x_{t_{j+1}} \mid x_{t_j}) &\propto q_\theta(x_{t_{j+1}} \mid x_{t_j}) \prod_{m=1}^M \left( \frac{q^*_m(x_{t_{j+1}} \mid x_{t_j})}{q_\theta(x_{t_{j+1}} \mid x_{t_j})} \right)^{\lambda_m} \\
    &\propto q_\theta(x_{t_{j+1}} \mid x_{t_j}) \exp\left( \sum_{m=1}^M \frac{\lambda_m A_m(x_{t_{j+1}})}{\tau} \right).
\end{aligned}
\label{eq:joint_target}
\end{equation}
Each ratio $q^*_m/q_\theta = \exp(A_m/\tau)/Z_m$ has a normalizer $Z_m$ that is independent of $x_{t_{j+1}}$, so the product $\prod_m Z_m^{-\lambda_m}$ is a constant absorbed into $\propto$. The $M$ separate tilts therefore collapse into a \emph{single} tilt of $q_\theta$ by the composite advantage $\sum_m \lambda_m A_m$, i.e.\ a lone target rather than $M$ competing ones.

Since each $A_m$ is affine in its normalized reward (Eq.~\ref{eq:advantage}), this joint tilt is realized exactly by fusing the $M$ normalized scores into a composite reward $r^{(k)}$ for each branch $k$:
\begin{equation}
    r^{(k)} = \sum_{m=1}^{M} \lambda_m\, \tilde{r}_m^{(k)}, \qquad \tilde{r}_m^{(k)} = \frac{r_m^{(k)} - \mu_m}{\sigma_m + \epsilon},
    \label{eq:composite_reward}
\end{equation}
where $\tilde{r}_m^{(k)}$ is the z-scored metric $m$ with weight $\lambda_m$: because an affine map preserves the branch ordering, ranking branches by $\sum_m \lambda_m \tilde{r}_m^{(k)}$ induces the same positive/negative sets as tilting by $\sum_m \lambda_m A_m$. A per-scorer veto rejects branches falling significantly below the self-reference baseline.

\textbf{Selector over trajectories, not objective over parameters.}
Crucially, the composite score $r^{(k)}$ enters training \textit{only} through the branch ranking that defines the advantages $A_k$; it is never differentiated. Consequently, the training target in Eq.~\ref{eq:loss_unified} remains a \textit{single, concrete trajectory velocity} $v_\pm^{(k)}$ that already scores well across all metrics at once, rather than a sum of per-metric gradients. This is where the two paradigms diverge at convergence. Field-level fusion combines objectives \emph{before} sampling, in the shared parameter space: when two objectives disagree their gradients point in opposing directions and each update is a compromise that fully satisfies neither, which in the teacher-based instantiation is realized by routing each prompt to the teacher whose objective it matches, so the aligned model carries high quality only on the prompt family that selected that teacher. Reward-level fusion combines objectives \emph{after} sampling: the composite score ranks whole trajectories and the winning branch already lies in the region that scores well under all rewards simultaneously. The search for a joint optimum thus moves from the parameter space, where gradients interfere, to the trajectory space, where a single sample can satisfy every reward at once. This reward-level paradigm yields several practical advantages: (i) it obviates inter-objective gradient conflicts, preventing metric collapse; (ii) it accommodates black-box or non-differentiable scorers; (iii) adjusting trade-offs via $\{\lambda_m\}$ requires zero model retraining, enabling runtime composability; and (iv) it bypasses the teacher ceiling through self-exploration, allowing the student to exceed any individual pre-trained teacher.

\textbf{Per-task fusion in practice.}
We instantiate reward-level fusion through Eq.~\ref{eq:composite_reward}. Because OCR and GenEval prompts are not interchangeable, we fuse rewards \emph{per task}: on the OCR task we set $\lambda_{\rm OCR}:\lambda_{\rm PickScore}:\lambda_{\rm HPSv2}=3:1:1$, and the GenEval variant replaces the task scorer analogously, both keeping PickScore and HPSv2 as shared prompt-agnostic quality guards against reward hacking. At each training step, every SDE-branch rollout image \emph{and} the pure-ODE self-reference baseline are scored by all three reward models; each scorer is z-normalized across this joint set, and the weighted sum forms the per-branch composite score from which the advantage (Eq.~\ref{eq:advantage}) is computed to rank branches. This single teacher-free run stays jointly competitive across all metrics.

\section{Experiments}
\label{sec:experiments}
 
\subsection{Experimental Settings}
\textbf{Implementation Details.} We use SD3.5-Medium~\cite{SD3} at 512×512 resolution as the base model.
We train with LoRA~\cite{LoRA} applied to the transformer.
We use $K = 8$ branches, noise level $\eta = 0.7$, and training timesteps per step $= 2$. We use the AdamW~\cite{kingma2014adam} optimizer with learning rate $3 \times 10^{-4}$.

\textbf{Tasks and Rewards.} We evaluate three tasks with corresponding reward models: (1)~\textit{text rendering}, scored by OCR character-level accuracy between the generated and target text; (2)~\textit{compositional generation}, scored by GenEval~\cite{GenEval}; and (3)~\textit{aesthetic and human-preference alignment}, scored by PickScore~\cite{PickScore} and HPSv2~\cite{HPSv2}.
We also train mixed-reward variants fusing each task scorer.

\textbf{Baselines.} We compare against three groups. \textit{Base models}: SD3.5-Medium~\cite{SD3}. \textit{Teacher-free RL fine-tuning}: Flow-GRPO~\cite{Flow-GRPO}, GRPO-Guard~\cite{GRPO-Guard}, and DiffusionNFT~\cite{DiffusionNFT}. \textit{Teacher-based OPD}: Flow-OPD~\cite{flow_opd} and DiffusionOPD~\cite{diffusion_opd}.

\begin{table*}[t]
\centering
\caption{\textbf{Main results of single-reward training.} Best in \textbf{bold}, second best \underline{underlined} among trained methods.}
\label{tab:main}
\begin{tabular}{llccccc}
\toprule
\multicolumn{2}{l}{\multirow{2}{*}{\textbf{Method}}} & \multicolumn{2}{c}{\textbf{GenEval}} & \multirow{2}{*}{\textbf{OCR}} &\multirow{2}{*}{\textbf{PickScore}} & \multirow{2}{*}{HPSv2}\\
\cmidrule(lr){3-4}
 & & strict & cont. & & & \\
\midrule
\multicolumn{7}{l}{\textit{Base models (no alignment fine-tuning)}} \\
\multicolumn{2}{l}{SD3.5-Medium (base)} & 0.5222 & 0.6219 & 0.5833 & 22.41 & 0.3004 \\
\midrule
\multicolumn{7}{l}{\textit{RL fine-tuning (teacher-free)}} \\
\multirow{3}{*}{Flow-GRPO} & + RM GenEval & 0.9005 & 0.9470 & 0.6519 & 22.53 & 0.2792 \\
 & + RM OCR & 0.5420 & 0.6529 & 0.9253 & 22.51 & 0.2999 \\
 & + RM PickScore & 0.5081 & 0.5293 & 0.7033 & 23.57 & 0.3396 \\
\cmidrule(lr){1-7}
\multirow{3}{*}{GRPO-Guard} & + RM GenEval & \underline{0.9155} & \underline{0.9502} & 0.7032 & 22.20 & 0.2586 \\
 & + RM OCR & 0.5402 & 0.6870 & \underline{0.9348} & 22.44 & 0.2912 \\
 & + RM PickScore & 0.4602 & 0.4037 & 0.6563 & \underline{23.98} & \underline{0.3453} \\
\midrule
\multicolumn{7}{l}{\textit{Ours (teacher-free)}} \\
\multirow{4}{*}{\textbf{Self-OPD}} & + RM GenEval & \textbf{0.9536} & \textbf{0.9676} & 0.6120 & 22.46 & 0.2741 \\
 & + RM OCR & 0.4186 & 0.5565 & \textbf{0.9745} & 22.11 & 0.2719 \\
 & + RM PickScore & 0.5303 & 0.5863 & 0.7416 & \textbf{24.47} & 0.3357 \\
 & + RM HPSv2 & 0.4453 & 0.4722 & 0.7121 & 23.25 & \textbf{0.4099} \\
\bottomrule
\end{tabular}%
\end{table*}

\begin{table*}[t]
\centering
\caption{\textbf{Main results of mixed-reward training.} Each method produces a single model evaluated across all metrics. }
\label{tab:main_mix}
\resizebox{\textwidth}{!}{%
\begin{tabular}{llcccccccc}
\toprule
\multicolumn{2}{l}{\multirow{2}{*}{\textbf{Method}}} & \multirow{2}{*}{\textbf{Teacher-free}} & \multicolumn{2}{c}{\textbf{GenEval}} & \multirow{2}{*}{\textbf{OCR }} & \multicolumn{2}{c}{\textbf{Pref. (same test images)}} & \multicolumn{2}{c}{\textbf{Pref. (separate test set)}} \\
\cmidrule(lr){4-5} \cmidrule(lr){7-8} \cmidrule(lr){9-10}
 & & & strict & cont. & & PickScore & HPSv2 & PickScore & HPSv2 \\
\midrule
\multicolumn{10}{l}{\textit{Base models (no alignment fine-tuning)}} \\
\multicolumn{2}{l}{SD3.5-Medium (base)} & --- & 0.5222 & 0.6219 & 0.5833 & 22.66 & 0.2824 & 22.41 & 0.3004 \\
\midrule
\multicolumn{10}{l}{\textit{Teacher-based OPD}} \\
\multicolumn{2}{l}{Flow-OPD} & \xmark & 0.8594 & 0.9203 & 0.9392 & 23.43 & 0.3042 & 23.13 & 0.3334 \\
\multicolumn{2}{l}{DiffusionOPD} & \xmark & \underline{0.9150} & \underline{0.9479} & \underline{0.9464} & 22.72 & 0.2676 & \textbf{23.95} & \textbf{0.3729} \\
\midrule
\multicolumn{10}{l}{\textit{Mixed reward (teacher-free)}} \\
\multicolumn{2}{l}{DiffusionNFT} & \cmark & 0.8888 & 0.9277 & 0.9229 & \underline{23.67} & \underline{0.3206} & 23.18 & \underline{0.3482} \\
\multicolumn{2}{l}{\textbf{Self-OPD}(Ours)} & \cmark & \textbf{0.9521} & \textbf{0.9691} & \textbf{0.9597} & \textbf{23.87} & \textbf{0.3214} & \underline{23.39} & 0.3415 \\
\bottomrule
\end{tabular}%
}
\end{table*}

\subsection{Main Results}
\textbf{Quantitative Results.}
We evaluate all methods across two distinct setups, namely single-reward training in Tab.~\ref{tab:main} and mixed-reward, multi-capability training in Tab.~\ref{tab:main_mix}. The assessment encompasses compositional generation via GenEval, text rendering via OCR, and human/AI preferences measured by PickScore and HPSv2.
For GenEval, we report both the strict mode, which binary-scores generations only when all sub-requirements are met, and the continuous mode, which allows for partial credit.
Similarly, preference metrics are assessed under two protocols: the same test images protocol that evaluates directly on the GenEval and OCR benchmarks, and the separate test set protocol that utilizes a held-out aesthetic benchmark.

Self-OPD emerges as the sole teacher-free method that remains highly competitive across all evaluation dimensions simultaneously.
In Tab.~\ref{tab:main}, among all single-reward methods, Self-OPD achieves top performance with a GenEval score of 0.95, OCR accuracy of 97.5\%, PickScore of 24.79, and HPSv2 score of 0.3665. Notably, Self-OPD consistently outperforms both Flow-GRPO and GRPO-Guard on every task despite utilizing identical base models and reward signals, validating the superiority of step-wise self-distillation over terminal policy gradient optimization.

\begin{figure*}[!t]
    \centering
    \includegraphics[width=\textwidth]{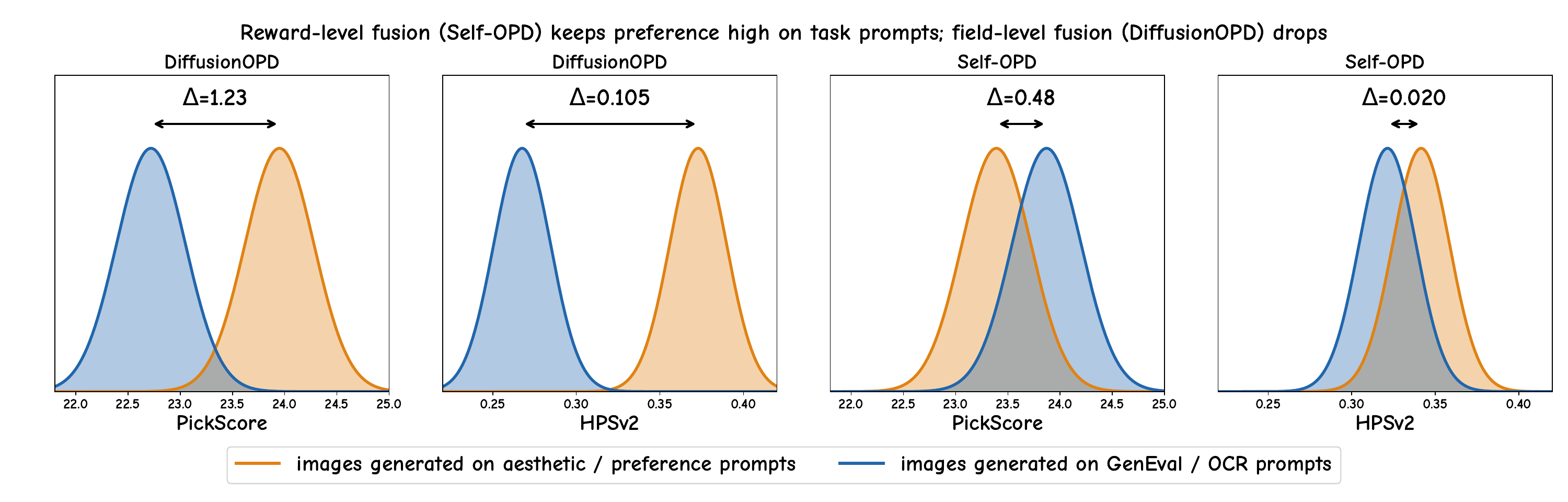}
    \caption{\textbf{Field-level fusion couples preference quality to the prompt family; reward-level fusion does not.} For each approach, we contrast the preference-reward distribution of images generated on aesthetic/preference prompts (orange) against that of images generated on the GenEval/OCR task prompts (blue). The distribution peaks are placed at the measured means from Tab.~\ref{tab:main_mix} (the ``separate test set'' and ``same test images'' columns, respectively); the spread is illustrative and only serves to visualize the shift $\Delta$ between the two prompt families. For DiffusionOPD, the preference distribution shifts sharply when moving to the task prompts, whereas for Self-OPD, the two distributions nearly coincide.}
    \label{fig:decouple_dist}
\end{figure*}

In Tab.~\ref{tab:main_mix}, the unified mixed-reward Self-OPD model shares the highest GenEval score of 0.95, delivers the best OCR accuracy of 96.0\%, and under the same test images protocol, outperforms the teacher-based DiffusionOPD on both PickScore, with a score of 23.87 compared to DiffusionOPD's 22.72, and HPSv2, with 0.3214 compared to 0.2676, despite requiring no teacher supervision. This performance gap empirically validates the reward-level versus field-level fusion paradigm.
Specifically, the core objective of multi-metric alignment requires each individual generation to satisfy all criteria simultaneously, rather than having different subsets of images excel under different evaluators. By selecting candidate SDE branches within the joint high-reward region, Self-OPD ensures that the same GenEval and OCR test images also exhibit superior aesthetic quality. Conversely, field-level fusion in DiffusionOPD routes each prompt to a specialized teacher, generating images that satisfy the isolated target objective but fail to jointly optimize general preferences. Consequently, its preference evaluations on the GenEval and OCR test images, yielding a PickScore of 22.72 and HPSv2 of 0.2676, fall below those of Flow-OPD, which scores 23.43 and 0.3042, respectively.

Fig.~\ref{fig:decouple_dist} makes this difference explicit. Reading each method as a distribution of preference rewards over two prompt populations, DiffusionOPD's PickScore distribution shifts down by $\Delta{=}1.23$ and its HPSv2 distribution by $\Delta{=}0.105$ when the images are generated on the GenEval/OCR task prompts rather than on aesthetic prompts. This is exactly the signature of field-level fusion described above: blending per-teacher gradients in parameter space yields images that score well only on the prompt family matched to the selected teacher, so preference quality is effectively decoupled from the task prompts. For Self-OPD the two distributions almost overlap ($\Delta{=}0.48$ for PickScore and $\Delta{=}0.020$ for HPSv2), because reward-level fusion selects trajectories inside the joint high-reward region rather than compromising between competing gradients. This is precisely why we treat the same test images protocol as the primary preference comparison: it measures whether the \emph{same} generations that satisfy the task also carry high aesthetic quality, the property multi-objective alignment actually targets.

\begin{figure*}[t]
    \centering
    \includegraphics[width=\textwidth]{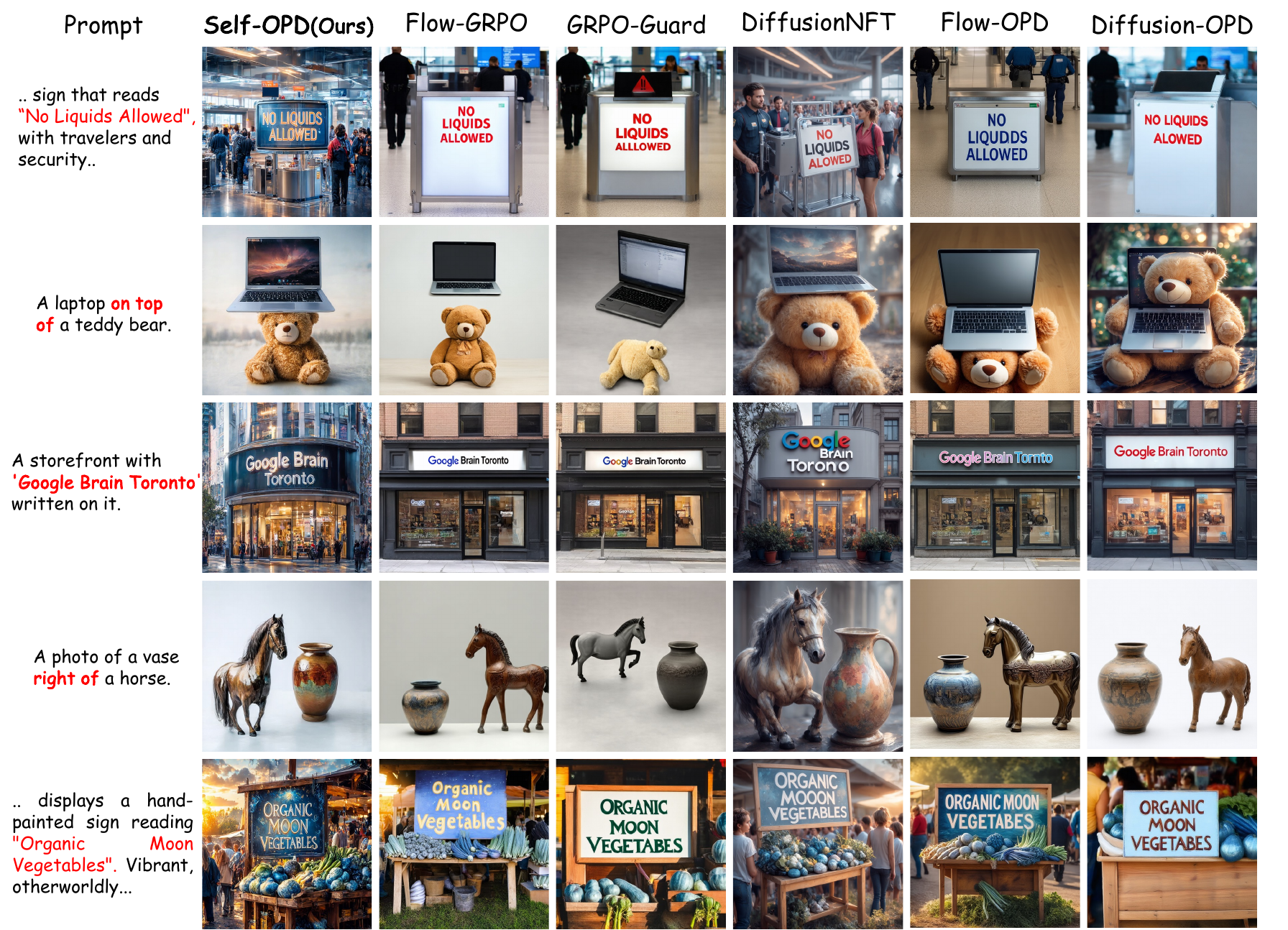}
    \caption{\textbf{Qualitative comparison.} Each method uses its mixed-reward model if available, or the prompt-specific reward model (e.g., OCR for text prompts). Self-OPD achieves superior performance in both accuracy and aesthetics.}
    \label{fig:comparison}
\end{figure*}

\textbf{Qualitative Results.} Fig.~\ref{fig:comparison} qualitatively compares all methods on six prompts spanning text rendering, spatial relations, and counting. First, among mixed-reward baselines including DiffusionNFT, Flow-OPD, and DiffusionOPD, Self-OPD renders text with higher fidelity, precisely spelling challenging phrases such as ``No Liquids Allowed'' and ``Google Brain Toronto''. It also accurately executes complex spatial layouts, such as placing a laptop on a teddy bear or a vase right of a horse, and strictly satisfies counting constraints like generating exactly four benches, yielding more coherent overall compositions.
Second, compared to single-reward specialists like Flow-GRPO and GRPO-Guard at their optimal checkpoints. Self-OPD exhibits superior control across all three axes. Driven by joint preference optimization, it concurrently enhances scene richness and aesthetic quality. These qualitative advantages closely align with the quantitative improvements in Tab.~\ref{tab:main} and Tab.~\ref{tab:main_mix}.

\FloatBarrier
\subsection{Ablation Study}
\label{sec:ablation}

\textbf{Effect of Branch Selection Strategy.}
Regressing solely to the top-performing branch (\textit{Best-of-$K$}, red) yields unstable, non-monotonic training that barely surpasses the baseline. This failure stems from high gradient variance caused by abrupt target switching and the complete absence of negative feedback. Conversely, our \textit{all-branch advantage-weighted} distillation (green) stabilizes training by utilizing the entire exploration neighborhood. Adding $|\Delta t|$-aligned KL normalization (blue; \textit{full Self-OPD}, Proposition~\ref{prop:kl}) further accelerates convergence and maximizes performance, validating our theoretical step-wise loss scaling.

\textbf{Effect of Repulsion Bounding.}
The scaling factor $\tfrac{1}{2}$ in Eq.~\ref{eq:direction} bounds the coefficient at $d_k \in [0,1]$, ensuring the repulsive push never overpowers the attractive pull. Under an unbounded formulation $d_k \in [0,2]$ (red), strong repulsive gradients from opposing branches eventually dominate, destabilizing training and triggering a sharp performance collapse. Our bounded gating (blue) maintains stable, monotonic improvement, confirming that the direction-aware coefficient must act purely as an attenuation gate rather than a gradient amplifier.

\textbf{Effect of Timestep Weighting.}
While uniform timestep weighting eventually matches our weighted scheme in final performance, our proposed weighting accelerates convergence by concentrating training signals on early timesteps, where branching yields more diverse semantic alternatives. Conversely, a timestep importance sampling (TIS) variant that oversamples early steps degrades performance below the base model, demonstrating that continuous exposure to all trajectory stages is essential even when loss weighting prioritizes early exploration phases.

\begin{figure}[!t]
    \centering
    \begin{minipage}[t]{0.45\linewidth}
        \centering
        \includegraphics[width=\linewidth]{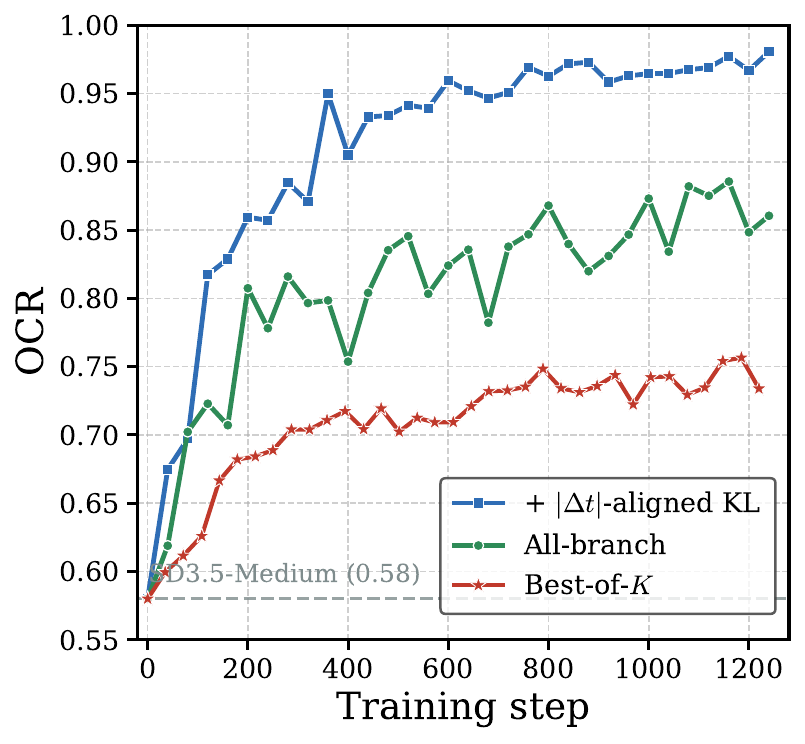}
        {\small (a) Branch selection strategy.}
        \label{fig:ablation_strategy}
    \end{minipage}
    \hfill
    \begin{minipage}[t]{0.45\linewidth}
        \centering
        \includegraphics[width=\linewidth]{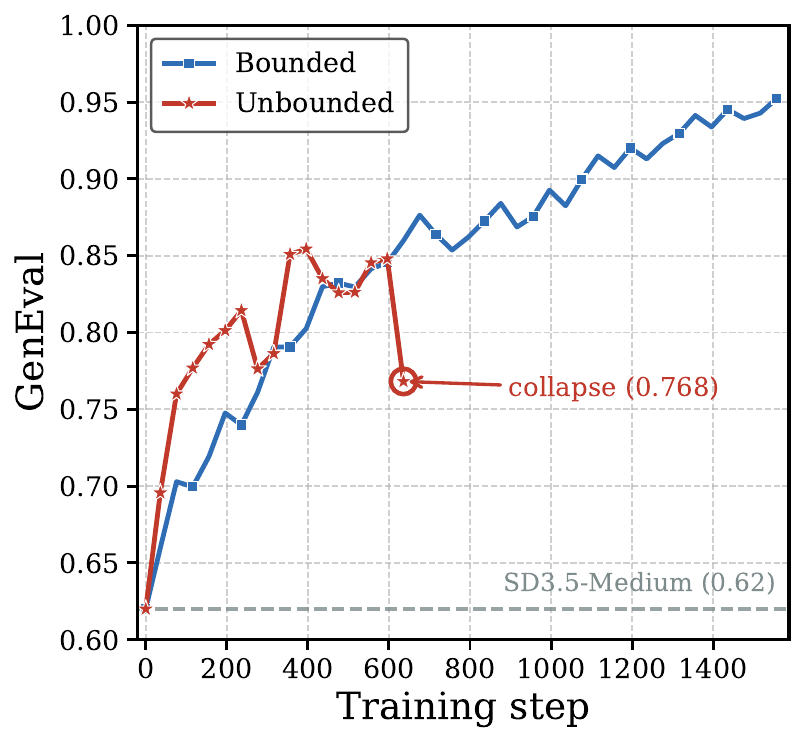}
        {\small (b) Bounded vs.\ unbounded.}
        \label{fig:ablation_repulsion}
    \end{minipage}    
    \caption{\textbf{Ablation of core designs.}
    \textbf{(a) Branch handling:} Best-of-$K$ (red), all-branch pull-push distillation (green), and Self-OPD with $|\Delta t|$-aligned KL (blue).
    \textbf{(b) Gradient gating:} Bounded $d_k$ (blue) vs.\ unbounded $d_k$ (red).}
    \label{fig:ablation}
\end{figure}

\begin{figure}[!t]
    \centering
    \includegraphics[width=0.49\textwidth]{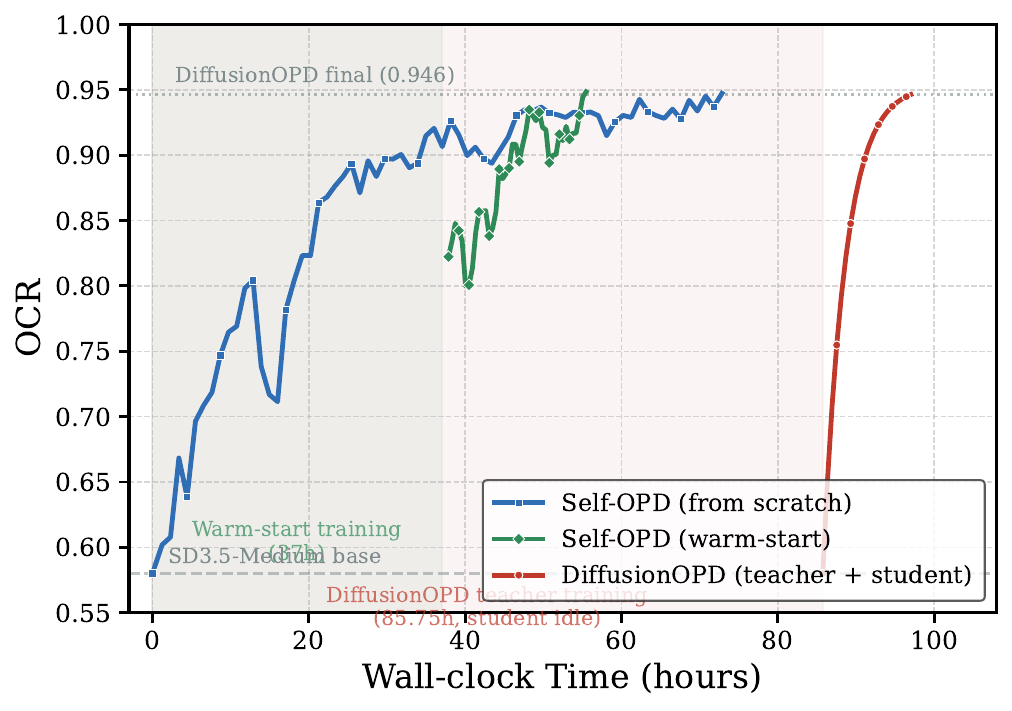}\hfill
    \includegraphics[width=0.49\textwidth]{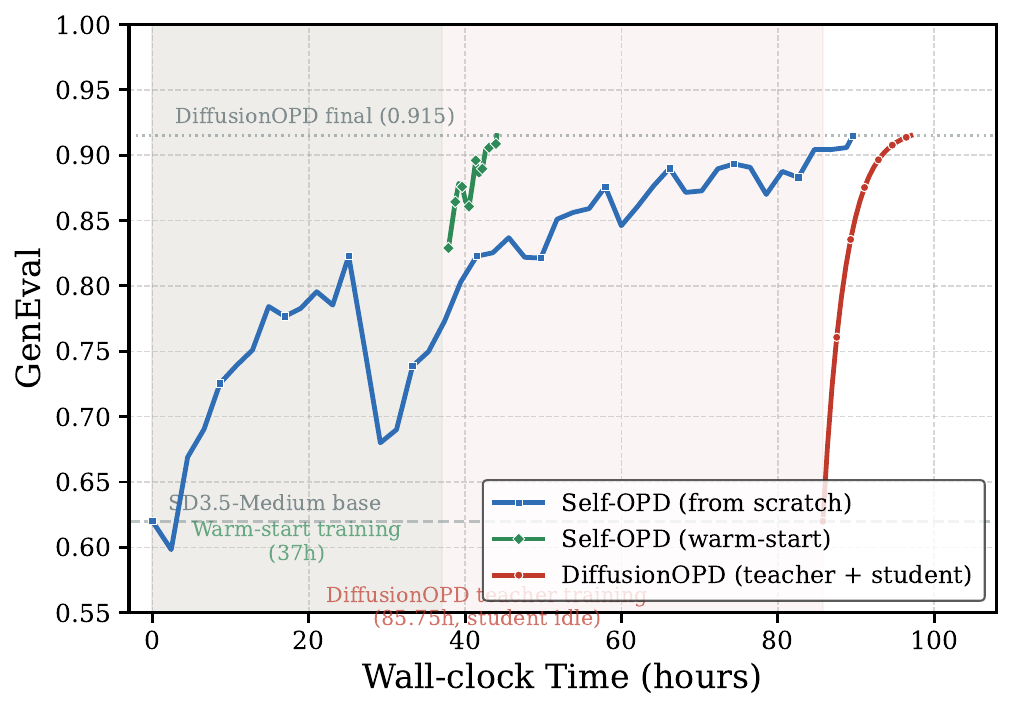}
    \caption{\textbf{Training efficiency: OCR (left) and GenEval (right) vs.\ wall-clock time.} Blue: Self-OPD mix from scratch. Green: Self-OPD mix warm-start. Red: DiffusionOPD student phase. The green shaded region indicates the time spent training warm-start checkpoints (37\,h). The red shaded region indicates DiffusionOPD's teacher training phase (85.75\,h), during which the student model is idle. The dotted horizontal line marks DiffusionOPD's final performance level. All training times for DiffusionOPD are taken from the original report.
    }
    \label{fig:training_efficiency}
\end{figure}

\textbf{Analysis of Training Time.} Fig.~\ref{fig:training_efficiency} makes Self-OPD's wall-clock advantage concrete on OCR and GenEval.
DiffusionOPD~\citep{diffusion_opd} requires three separate teacher models (GenEval: 46.9\,h, OCR: 33.2\,h, Aesthetics: 85.8\,h) trained in parallel, followed by 11.3\,h of student distillation, for a total wall-clock time of $85.8 + 11.3 = 97.0$\,h.
During the entire teacher-training phase (red shaded region), the student remains at the base level and cannot begin learning.
Self-OPD eliminates this bottleneck entirely, in two configurations:

\begin{itemize}
    \item \textit{From scratch} (blue): a single tri-reward run that begins improving immediately. It reaches DiffusionOPD-level OCR (0.946) in $\sim$62\,h and GenEval (0.915) in $\sim$90\,h, both strictly faster than DiffusionOPD's 97\,h.
    \item \textit{Warm-start} (green): first trains single-reward specialists in parallel (GenEval: 27\,h, OCR: 37\,h; wall-clock: 37\,h, green shaded region), then launches the tri-reward mixed run which converges to 0.946 OCR in $\sim$11\,h and 0.915 GenEval in $\sim$7\,h, for a total wall-clock of $\sim$48\,h and $\sim$44\,h respectively. This is roughly $2\times$ faster than DiffusionOPD on both metrics, while achieving higher final performance.
\end{itemize}

\section{Conclusion}
In this paper, we introduce Self-OPD, a teacher-free on-policy distillation framework for flow matching models.
By branching the student's own trajectory, scoring each branch with task rewards, and applying all-branch advantage-weighted regression with direction-aware modulation, Self-OPD turns self-exploration into dense per-step supervision.
Experiments across text rendering, compositional generation, and preference alignment show that this formulation can match or surpass teacher-based OPD while avoiding task-specific teacher training and enabling reward-level multi-objective fusion.

\clearpage
\bibliographystyle{plainnat}
\setlength{\bibhang}{0pt}
\setlength\bibindent{0pt}
\bibliography{main}

\clearpage

\appendix
\setcounter{secnumdepth}{2}

This supplementary material provides additional evaluation details, a theoretical proof, extended analysis, and qualitative results referenced in the main paper. The contents are organized as follows:

\begin{itemize}
\item Appendix~\ref{app:metrics} documents the evaluation metrics, including the two GenEval scoring modes and the two preference evaluation protocols.
\item Appendix~\ref{app:qualitative} gives additional qualitative results demonstrating the aesthetic advantage of Self-OPD on the task-specific prompts.
\end{itemize}

\section{Evaluation Metrics}
\label{app:metrics}

\subsection{GenEval: Strict vs.\ Continuous Scoring}
\label{app:geneval}

GenEval~\cite{GenEval} evaluates compositional generation by checking whether generated images satisfy a set of sub-requirements (e.g., correct object count, spatial relation, color attribute). We report two scoring modes:

\begin{itemize}
    \item \textbf{Strict scoring.} A binary per-image metric: the score is 1 only if \emph{all} sub-requirements for that prompt are satisfied, and 0 otherwise. The reported number is the average over all test prompts. This is a demanding metric that penalizes any partial failure.
    \item \textbf{Continuous scoring.} A partial-credit metric: for each image, the score is the fraction of sub-requirements satisfied (ranging from 0 to 1). The reported number is again averaged over all test prompts. This gives credit to images that partially fulfill the prompt.
\end{itemize}

\begin{figure}[t]
    \centering
    \includegraphics[width=0.5\linewidth]{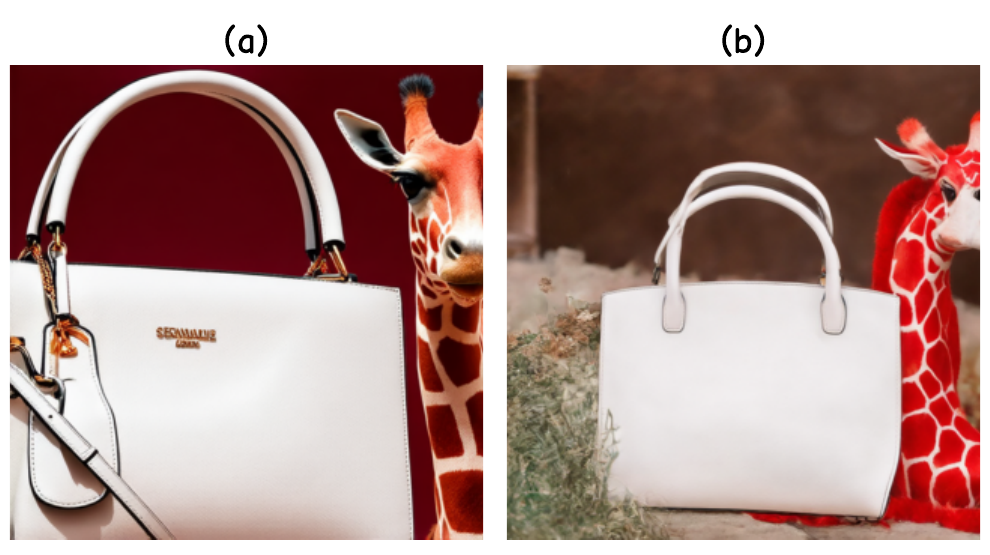}
    \caption{\textbf{Strict vs.\ continuous scoring} on ``a photo of a white handbag and a red giraffe''. (a) base; (b) Self-OPD.}
    \label{fig:strict_cont}
    \vspace{1em}
    \centering
    \begin{tabular}{lcc}
        \toprule
        Model & Strict & Continuous \\
        \midrule
        SD3.5-Medium (base) & 0 & 0.5 \\
        Self-OPD & 1 & 1.0 \\
        \bottomrule
    \end{tabular}
    \captionof{table}{Strict and continuous GenEval scores for the two images in Fig.~\ref{fig:strict_cont}. The prompt has two sub-requirements (one per colored object).}
    \label{tab:strict_cont}
\end{figure}

Fig.~\ref{fig:strict_cont} illustrates both modes on the color-attribution prompt ``a photo of a white handbag and a red giraffe'', which has two sub-requirements, one per colored object. The SD3.5-Medium base (Fig.~\ref{fig:strict_cont}a) renders the white handbag correctly but colors the giraffe orange rather than red, satisfying one of the two sub-requirements: its strict score is $0$ while its continuous score is $0.5$. Self-OPD (Fig.~\ref{fig:strict_cont}b) renders both the white handbag and the red giraffe, satisfying both sub-requirements for a strict score of $1$ and a continuous score of $1.0$ (Tab.~\ref{tab:strict_cont}). Continuous scoring better captures incremental improvements, while strict scoring reflects end-to-end correctness.

\subsection{Preference Protocols: Same Test Images vs.\ Separate Test Set}
\label{app:preference}

We evaluate preference-based metrics (PickScore~\cite{PickScore} and HPSv2~\cite{HPSv2}) under two complementary protocols:

\begin{itemize}
    \item \textbf{Same test images.} PickScore and HPSv2 are computed on the \emph{same} images generated for the GenEval and OCR test prompts. Specifically, for each model we generate images for the GenEval test suite and the OCR test suite, then score all these images with PickScore and HPSv2. The final score is the unweighted average of the GenEval-prompt score and the OCR-prompt score (1:1 averaging). This protocol measures whether a model maintains aesthetic quality and human preference \emph{on the task it was trained for}, without introducing any distribution shift from additional prompts.
    \item \textbf{Separate test set.} PickScore and HPSv2 are computed on a held-out set of aesthetic-oriented prompts (DrawBench) that are disjoint from the training and GenEval/OCR evaluation prompts. This measures preference quality on a general-purpose prompt distribution.
\end{itemize}

The ``same test images'' protocol gives a more controlled cross-metric comparison because it evaluates all metrics on the same set of generations, separating model quality from prompt distribution. A model with high preference scores under this protocol improves aesthetic quality without sacrificing task performance. The ``separate test set'' protocol can instead favor models that overfit to aesthetic-leaning prompts while degrading on structured tasks. For image generation we care about the overall quality of every generated image rather than separate high scores on different task sets, so we treat the ``same test images'' protocol as the primary preference comparison. Fig.~\ref{fig:decouple_dist} in Appendix~\ref{app:qualitative} makes this concrete: DiffusionOPD's lead on the separate set does not carry over to the GenEval and OCR test images, because its field-level fusion couples preference quality to the prompt family, whereas Self-OPD stays high on both.

\begin{figure}[!t]
    \centering
    \includegraphics[width=\textwidth]{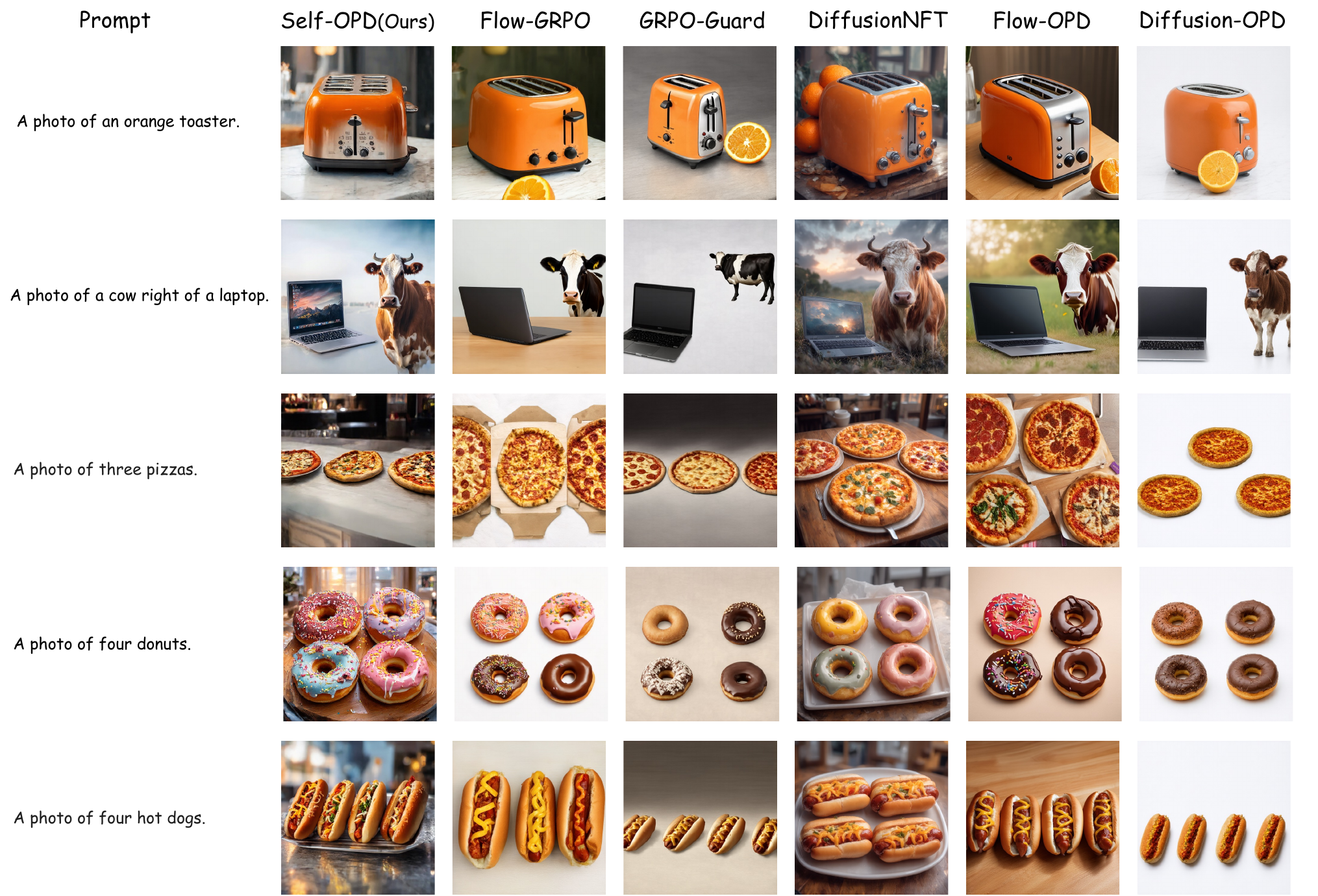}
    \caption{\textbf{Aesthetic quality on GenEval prompts.} Each method uses its mixed-reward model if available, or its GenEval-trained model otherwise (Flow-GRPO, GRPO-Guard). Self-OPD satisfies the compositional requirements (counting, spatial relations, attributes) while generating images with higher aesthetic quality, consistent with its higher PickScore/HPSv2 in Tab.~\ref{tab:main_mix}.}
    \label{fig:aesthetic_geneval}
\end{figure}

\begin{figure}[!t]
    \centering
    \includegraphics[width=\linewidth]{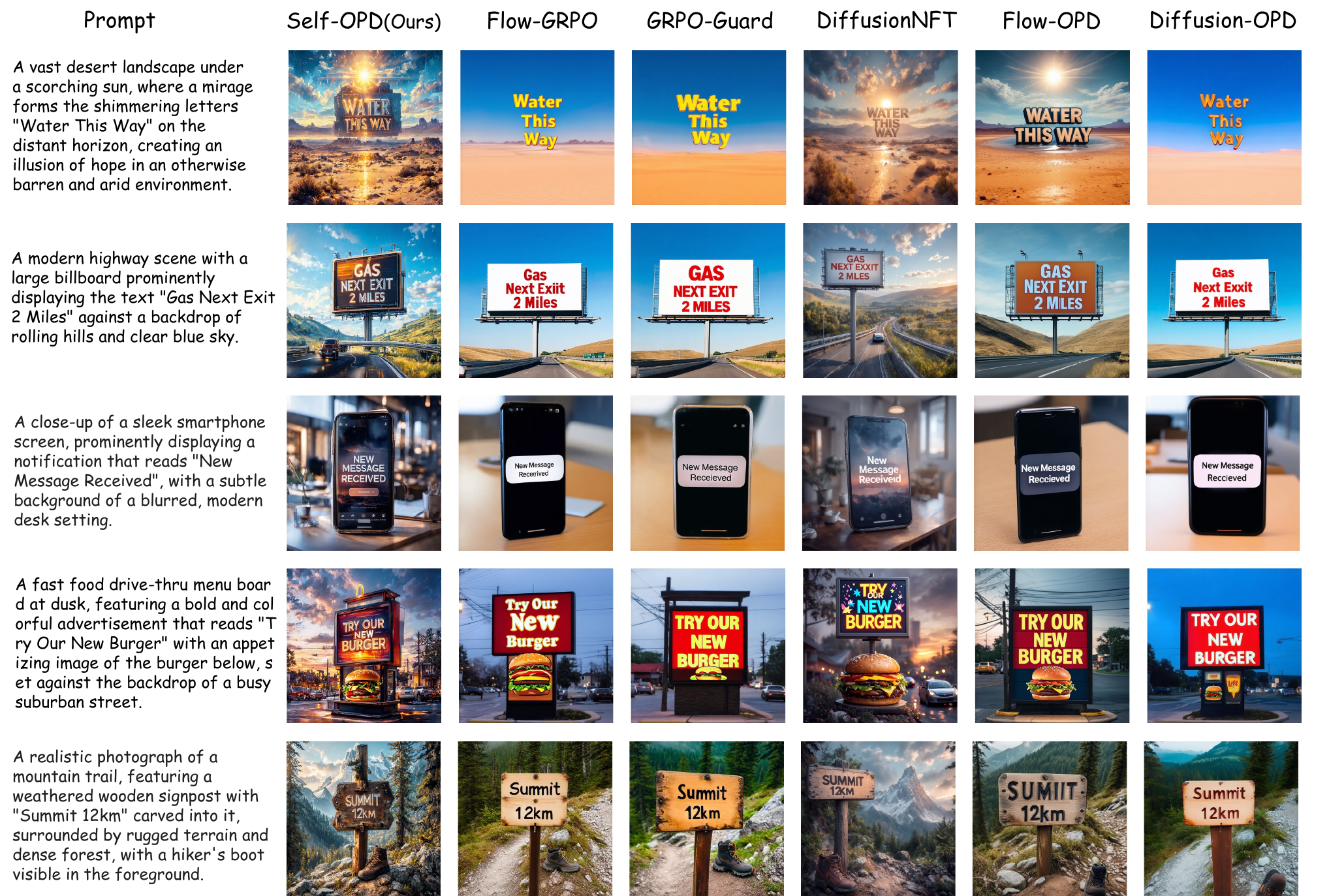}
    \captionof{figure}{\textbf{Aesthetic quality on OCR prompts.} Each method uses its mixed-reward model if available, or its OCR-trained model otherwise (Flow-GRPO, GRPO-Guard). Self-OPD renders the target text accurately while producing more visually appealing scenes with richer backgrounds and realistic details.}
    \label{fig:aesthetic_ocr}
\end{figure}

\section{Additional Qualitative Results}
\label{app:qualitative}

Fig.~\ref{fig:aesthetic_geneval} and Fig.~\ref{fig:aesthetic_ocr} visualize the aesthetic quality gap on task-specific prompts, comparing Self-OPD against Flow-GRPO, GRPO-Guard, DiffusionNFT, Flow-OPD, and DiffusionOPD. Under the same test images protocol, Self-OPD attains the highest PickScore and HPSv2 among all methods (Tab.~\ref{tab:main_mix}); these figures illustrate why. On GenEval prompts (Fig.~\ref{fig:aesthetic_geneval}), Self-OPD produces images with richer lighting, more natural textures, and better composition while satisfying the compositional requirements.

As shown in Fig.~\ref{fig:aesthetic_ocr}, Self-OPD renders text accurately and simultaneously generates more visually appealing scenes with coherent backgrounds and realistic details, whereas other methods tend to produce flatter or overly simplistic images.
Given identical prompts, Self-OPD maintains
high aesthetic quality, while DiffusionOPD and the RL baselines exhibit visible degradation, e.g., over-saturation and loss of detail.

\end{document}